\documentclass{article}

\PassOptionsToPackage{numbers,sort&compress}{natbib}  
\usepackage[preprint]{neurips_2026}  

\usepackage[utf8]{inputenc}
\usepackage[T1]{fontenc}
\usepackage{hyperref}
\usepackage{url}
\usepackage{booktabs}
\usepackage{amsfonts}
\usepackage{amsmath}
\usepackage{amssymb}
\usepackage{amsthm}
\usepackage{nicefrac}
\usepackage{microtype}
\usepackage{xcolor}
\usepackage{verbatim}
\usepackage{graphicx}
\usepackage{algorithm}
\usepackage{algorithmic}
\usepackage{multirow}
\usepackage{makecell}
\usepackage{subcaption}
\usepackage[capitalize,noabbrev]{cleveref}

\usepackage{aliascnt}

\newtheorem{theorem}{Theorem}

\newaliascnt{proposition}{theorem}

\aliascntresetthe{proposition}

\newaliascnt{corollary}{theorem}

\aliascntresetthe{corollary}

\newaliascnt{lemma}{theorem}
\newtheorem{lemma}[lemma]{Lemma}
\aliascntresetthe{lemma}

\newaliascnt{definition}{theorem}

\aliascntresetthe{definition}

\newaliascnt{remark}{theorem}

\aliascntresetthe{remark}

\newtheorem{assumption}{Assumption}

\crefname{assumption}{Assumption}{Assumptions}
\crefname{theorem}{Theorem}{Theorems}
\crefname{proposition}{Proposition}{Propositions}
\crefname{corollary}{Corollary}{Corollaries}
\crefname{lemma}{Lemma}{Lemmas}
\crefname{definition}{Definition}{Definitions}
\crefname{remark}{Remark}{Remarks}

\Crefname{assumption}{Assumption}{Assumptions}
\Crefname{theorem}{Theorem}{Theorems}
\Crefname{proposition}{Proposition}{Propositions}
\Crefname{corollary}{Corollary}{Corollaries}
\Crefname{lemma}{Lemma}{Lemmas}
\Crefname{definition}{Definition}{Definitions}
\Crefname{remark}{Remark}{Remarks}

\newcommand{\R}{\mathbb{R}}
\newcommand{\E}{\mathbb{E}}
\newcommand{\N}{\mathcal{N}}

\DeclareMathOperator{\Tr}{tr}
\DeclareMathOperator{\divop}{div}
\DeclareMathOperator{\TV}{TV}
\DeclareMathOperator{\KL}{KL}

\newcommand{\vf}{\vec{f}_\theta}           
\newcommand{\vg}{\vec{g}_\psi}             
\newcommand{\vfaug}{\vec{f}_\theta^{\mathrm{aug}}}
\newcommand{\flow}[3]{\Phi_{#1,\,#2 \to #3}} 

\newcommand{\ptilde}[1][]{\tilde{p}_{\theta#1}}     
\newcommand{\ptheta}[1][]{p_{\theta#1}}             
\newcommand{\qpsi}[1][]{q_{\psi#1}}              

\newcommand{\Lma}{\mathcal{L}_{\mathrm{MA}}}

\newcommand{\DeltaStu}{\Delta^{\mathrm{Stu}}}
\newcommand{\Xstu}{X^{\mathrm{Stu}}}
\newcommand{\spsi}{s_\psi}                

\newcommand{\stopgrad}{\mathrm{stopgrad}}
\newcommand{\dd}{\mathrm{d}}

\newcommand{\cF}{{\cal F}}
\newcommand{\cL}{{\cal L}}

\title{Beyond Trajectory Matching: Reflow with Marginal Distribution Alignment}

\author{
  Chen Wang\thanks{Equal contribution.} \\
  Department of Statistics and Data Science, Tsinghua University \\
  \texttt{wangchen\_23@mails.tsinghua.edu.cn}
  \And
  Peiran Yun\footnotemark[1] \\
  Department of Computer Science and Technology, Tsinghua University \\
  \texttt{yunpr22@mails.tsinghua.edu.cn}
  \And
  Pan Xie \\
  Bytedance Inc \\
  \texttt{xiepan.01@bytedance.com}
  \And
  Ke Deng\thanks{Corresponding author.} \\
  Department of Statistics and Data Science, Tsinghua University \\
  \texttt{kdeng@tsinghua.edu.cn}
}

\begin{document}

\maketitle

\begin{abstract}
Diffusion and continuous-flow generative models achieve high-quality generation, and their deterministic sampling can be formulated as solving learned ODE dynamics. However, accurate ODE discretization often requires many steps, making efficient few-step generation a key challenge. Among acceleration strategies, reflow-based distillation simplifies teacher ODE trajectories so that a student model can approximate the teacher transport with fewer steps. We identify a theoretical limitation of this paradigm, namely that trajectory matching can under-determine the distribution induced by the student model. In particular, two student models can attain the same trajectory-matching loss while inducing different endpoint marginal distributions, which may lead to different generation quality. To address this limitation, we introduce a marginal-alignment regularizer that penalizes the discrepancy between the student-induced marginal and the corresponding teacher marginal at the endpoint of each distillation interval. The regularizer is computed by tracking log-density changes along the ODE induced by the student model and evaluating scores from the frozen teacher model, without requiring auxiliary trainable networks or adversarial optimization. The resulting framework applies uniformly to the reflow family, including vanilla reflow and piecewise reflow. We further prove a telescoping total-variation bound showing that local marginal alignment controls the final-time discrepancy between the student-induced and teacher-induced distributions. Experiments on benchmark backbones demonstrate the effectiveness of the proposed method for few-step generation.

\end{abstract}

\section{Introduction}
\label{sec:intro}

In recent years, diffusion models~\citep{ho2020ddpm,song2021sde,song2021ddim} and \emph{continuous normalizing flow} (CNF) models~\citep{chen2018neural,lipman2022flow,liu2022flow} have achieved remarkable progress in multimodal generation, enabling broad applications such as text-to-image generation~\citep{rombach2022ldm,podell2023sdxl}, text-to-video generation~\citep{ho2022video,blattmann2023align}, and image editing~\citep{meng2022sdedit,brooks2023instructpix2pix}.
A typical CNF-based generative model gradually transports samples drawn from a simple initial distribution to samples from the target distribution through an \emph{ordinary differential equation} (ODE) with a learned or designed velocity field~\citep{chen2018neural}.
Although diffusion models are usually formulated through stochastic noising and denoising processes, their deterministic sampling can be performed by solving probability flow ODEs (PF-ODEs)~\citep{song2021sde}. This ODE view provides a common language for discussing both model families.

In practice, however, ODE dynamics often vary substantially across time and space, so coarse discretization can introduce large numerical error and noticeably degrade sample quality. 
To reduce the number of sampling steps while preserving sample quality, many efforts have been made from this ODE-based sampling perspective.
Among common solutions, training-free sampler design~\citep{song2021ddim,lu2022dpm} attempts to improve the numerical solver without additional training, whereas distillation~\citep{salimans2022progressive,liu2022flow,song2023consistency} adapts the model and often provides better quality under very small step budgets.

Within the distillation route, reflow-style methods provide a principled paradigm for improving few-step sampling from generative ODEs by simplifying teacher ODE trajectories. Reflow~\citep{liu2022flow} introduced this paradigm by learning ODE dynamics that connect the endpoints of teacher trajectories with simpler paths, so that accurate sampling can be achieved with fewer Euler steps. Piecewise variants such as PeRFlow~\citep{yan2024perflow} extend this principle by partitioning the full time interval into several sub-intervals and simplifying trajectories within each window, thereby approaching piecewise linear probability flows. These methods have made few-step generation more effective, yet their trajectory-level objectives leave open a basic question: \emph{does trajectory-matching quality fully determine the quality of the induced marginals?} This question is central because generation quality ultimately depends on the induced marginals rather than merely on trajectory-level accuracy.

Our analysis gives a negative answer to this question. In practical distillation settings, velocity matching is typically imperfect due to finite model capacity and limited optimization. We prove that, under such imperfect matching, trajectory-level error does not fully determine the quality of the induced marginal distributions. We formalize this limitation in \cref{thm:underdetermination}: two student velocity fields can attain \emph{identical} trajectory MSE while inducing different marginal distributions at the interval endpoint, which can translate into different sample quality.

Motivated by this observation, we introduce a \emph{marginal-alignment regularizer} for reflow-style distillation. The regularizer directly aligns the student-induced marginal with the corresponding teacher marginal at each time-interval endpoint, and applies uniformly to the reflow family, with vanilla reflow and PeRFlow as special cases. As \Cref{fig:hero} conceptually illustrates, the regularizer complements the trajectory-matching loss by adding a distribution-level signal that favors better endpoint marginal alignment. We further show that the local marginal discrepancies penalized by the regularizer control the final-time discrepancy between the student-induced and teacher-induced distributions through a telescoping total-variation (TV) bound (\cref{thm:tvbound}). The computation of the regularizer relies only on log-density tracking along the ODE induced by the student model and score evaluations from the frozen teacher model, requiring no auxiliary trainable networks or adversarial optimization. Empirically, augmenting reflow-based distillation with marginal alignment consistently improves few-step generation on benchmark backbones such as SD~1.5 and SDXL.

In summary, our contributions are as follows:
\begin{enumerate}
    \item We identify a distribution-level under-determination in reflow-style distillation under imperfect velocity matching, and demonstrate it through a constructive example.
    \item We propose a marginal-alignment regularizer for the reflow family, which encourages interval-endpoint marginal alignment through an augmented ODE solve for the student dynamics and frozen-teacher score evaluations, without introducing additional trainable parameters.
    \item We provide a theoretical guarantee for marginal alignment, showing that interval-wise marginal discrepancies control the final-time discrepancy between the student-induced and teacher-induced distributions via a telescoping total-variation bound.
\end{enumerate}

\begin{figure}[t]
    \centering
    \includegraphics[width=\linewidth]{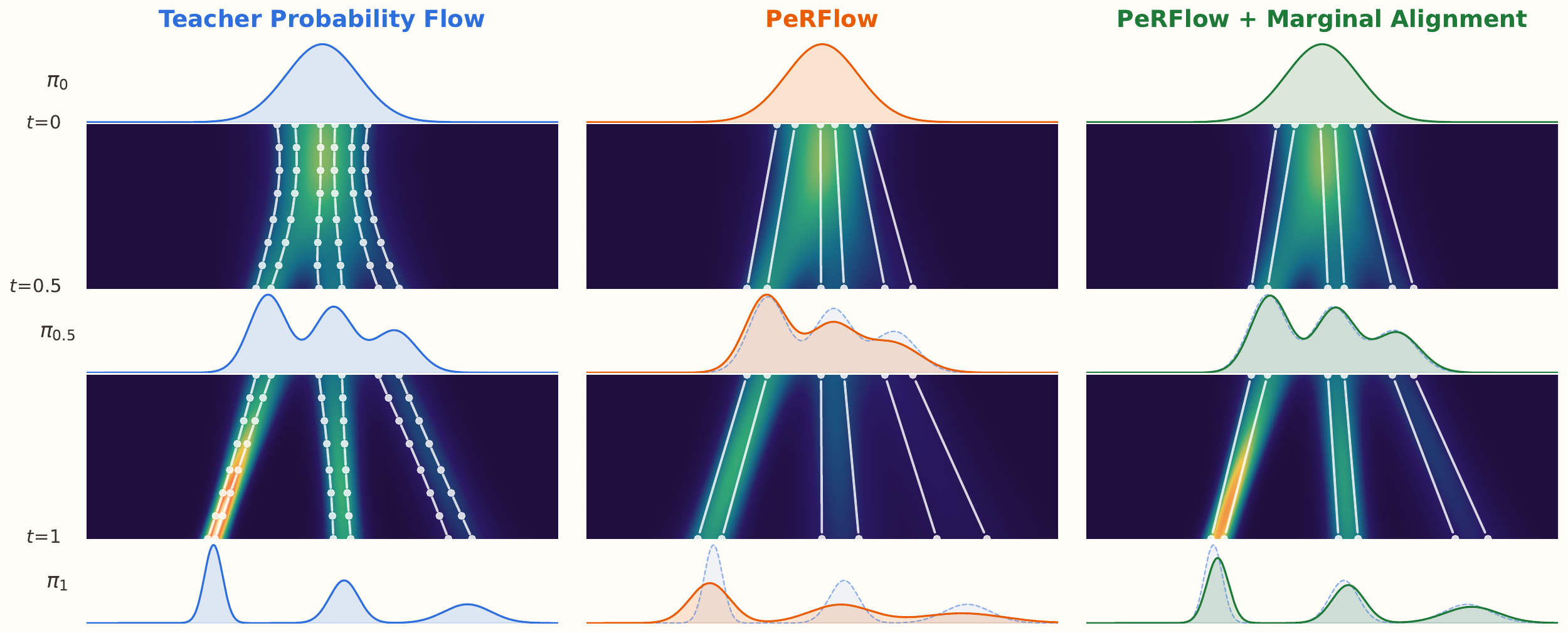}
    \caption{
    Conceptual overview on a 2D mixture-of-Gaussians setting.
    The PeRFlow student simplifies trajectories within each sub-interval but may leave boundary marginals misaligned, whereas adding marginal alignment encourages better agreement with the teacher marginals while keeping trajectories simple.
    Dashed curves denote teacher marginals.
    }
    \label{fig:hero}
\end{figure}

\section{Related Works}
\label{sec:related}

\subsection{Continuous Normalizing Flows}
Let $\pi_0 = \N(0,I_d)$ be the initial noise distribution, and $\pi_1$ be the data distribution.
A CNF evolves $\pi_0$ into a model distribution close to $\pi_1$ under deterministic ODE dynamics.
Given an initial random variable $X_0 \sim \pi_0$ and a time-varying velocity field $\vec{v}: \R^d \times [0,1] \to \R^d$, the trajectory $\{X_\tau\}_{\tau\in[0,1]}$ evolves according to
\begin{equation}\label{eq:ODE}
    \frac{\dd X_\tau}{\dd \tau} = \vec{v}(X_\tau,\tau), \qquad \tau\in[0,1].
\end{equation}
For any $0\leq s \leq t \leq 1$, solving the ODE in Equation~\eqref{eq:ODE} from the initial state-time pair $(x,s)$ to time $t$ yields a flow map $\flow{\vec{v}}{s}{t}$:
\begin{equation}\label{eq:FlowMap}
    \flow{\vec{v}}{s}{t}(x) := X_t,
    \qquad
    \frac{\dd X_\tau}{\dd \tau}
    =
    \vec{v}(X_\tau,\tau),
    \quad
    X_s=x,
    \quad \tau\in[s,t].
\end{equation}
For a neural velocity field $\vg(x,t)$ with parameters $\psi$, we denote the induced marginal at time $t$ along the CNF path by $q_{\psi,t} := \bigl(\flow{\vg}{0}{t}\bigr)_{\#}\pi_0$, where $(T)_{\#}\pi$ denotes the distribution of $T(X)$ for $X \sim \pi$.
In particular, $q_{\psi,0}=\pi_0$, and $q_{\psi,1}$ is the approximated model distribution by velocity field $\vg$.

When $q_{\psi,t}$ admits a density, its evolution is governed by the Liouville equation,
\begin{equation}
    \partial_\tau q_{\psi,\tau}(x) + \nabla_x \cdot \bigl(q_{\psi,\tau}(x)\vg(x,\tau)\bigr) = 0,
\end{equation}
which implies the instantaneous change-of-variables formula along a trajectory~\citep{chen2018neural,grathwohl2019ffjord}:
\begin{equation}
    \frac{\dd}{\dd \tau}\log q_{\psi,\tau}(X_\tau) = -\divop_x \vg(X_\tau,\tau),
    \label{eq:instantaneous_change_of_variables}
\end{equation}
where $\divop_x \vg$ is the trace of the Jacobian $\nabla_x \vg$.

In practice, the velocity field $\vg(x,t)$ modeled by a deep neural network with parameter $\psi$ is trained on a collection of representative samples from $\pi_1$ using objectives such as flow matching~\citep{lipman2022flow}. 
After training, the resulting model distribution is expected to approximate the data distribution, namely $q_{\psi,1} \approx \pi_1$.
Given the learned velocity field $\vg$, a sample $Y$ from the model distribution $q_{\psi,1}$ can be generated from a sample $Z\sim\pi_0$ by following the ODE dynamics in~\cref{eq:ODE}. 
With an $M$-step discretization of the time interval $[0,1]$, this trajectory takes the form
\begin{equation}\label{eq:trajectory}
Z=X_0\rightarrow X_{t_1}\rightarrow X_{t_2} \rightarrow \cdots \rightarrow  X_{t_{M-1}} \rightarrow  X_1=Y,
\end{equation}
where the symbol `$X_{s}\rightarrow X_{t}$' stands for a single Euler step:
$$X_{t}=X_{s}+\vg(X_{s},s)\cdot(t-s).$$

\subsection{Distillation by Piecewise Reflow}
\label{sec:perflow}
Given a pretrained CNF model with velocity field $\vg$, a more efficient CNF model with velocity field $\vf$ parameterized by $\theta$ can be obtained by distillation.
Hereinafter, we refer to $\vg$ and $\vf$ as the teacher and student velocity fields, respectively.
In distillation, the student is trained to define alternative ODE dynamics that approximate the teacher transport with simpler trajectories and fewer sampling steps.
A popular distillation strategy is Piecewise Reflow (PeRFlow)~\citep{yan2024perflow}, which partitions the time interval $[0,1]$ into $K$ windows ($K\ll M$) with cutting points $0 = t_0 < t_1 < \cdots < t_K = 1$, and trains $\vf$ by the following total \emph{trajectory matching loss}:
\begin{equation}\label{eq:TrajectoryMatchingLoss}
    \cL_{\mathrm{TM}}(\theta)
    =
    \sum_{k=1}^K \cL_{\mathrm{TM}}^{(k)}(\theta),
\end{equation}
where the $k$-th window-wise component is defined as
\begin{equation}
    \cL_{\mathrm{TM}}^{(k)}(\theta)
    =
    \E_{X \sim \pi_{t_{k-1}}}
    \int_{t_{k-1}}^{t_k}
    \left\|
    \vf\left(
    X + \frac{\tau - t_{k-1}}{t_k - t_{k-1}}(Y - X),\, \tau
    \right)
    -
    \frac{Y - X}{t_k - t_{k-1}}
    \right\|^2
    \dd \tau,
    \label{eq:perflow}
\end{equation}
where $\pi_t$ is the marginal of $X_t = a(t) X_0 + b(t) X_1$ with independent $X_0 \sim \pi_0$, $X_1 \sim \pi_1$ and a noise schedule $\{(a(t), b(t))\}_{t\in[0,1]}$ satisfying $a(0)=1, b(0)=0, a(1)=0, b(1)=1$, and $Y:= \flow{\vg}{t_{k-1}}{t_k}(X)$ is the endpoint of the $\vg$-determined trajectory starting from $X$ in the time window $W_k=[t_{k-1}, t_k]$.
The trajectory matching loss $\cL_{\mathrm{TM}}^{(k)}(\theta)$ encourages $\vf$ to match the constant velocity along the straight line connecting $X$ and $Y$ within each time window $W_k=[t_{k-1},t_k]$. 
This piecewise-linear approximation allows PeRFlow to reduce the number of sampling steps from $M$ to $K$ while largely preserving sample quality. 
When $K=1$, PeRFlow recovers the vanilla reflow procedure of Rectified Flow~\citep{liu2022flow}.

\subsection{Other Distillation Methods}
\label{sec:related_works}

\paragraph{Other reflow-family methods.}
Reflow-based distillation originates from Rectified Flow~\citep{liu2022flow}, which first proposed the trajectory-matching objective for learning straighter ODE dynamics. InstaFlow~\citep{liu2023instaflow} scaled this idea to text-to-image generation, and PeRFlow~\citep{yan2024perflow} extended it to the piecewise setting as reviewed in \cref{sec:perflow}. Subsequent work refined rectified-flow training~\citep{lee2024improving}, and Rectified Diffusion~\citep{wang2025rectified} showed that reflow is compatible with arbitrary noise schedules and prediction parameterizations. 
A major limitation of these methods lies in the fact that they rely on trajectory supervision alone and may under-determine the induced marginals (see \cref{thm:underdetermination} for details).

\paragraph{Consistency distillation.}
Consistency Models~\citep{song2023consistency} learn to map any point along an ODE trajectory directly to the endpoint at $t{=}1$, enabling one-step or few-step sampling. Following works have extended this idea to latent diffusion models~\citep{luo2023latent}, piecewise consistency constraints~\citep{wang2024pcm}, flexible start-to-end mappings~\citep{kim2024ctm,frans2025shortcut}, and improved training stability at scale~\citep{lu2025scm,ren2024hypersd}. These methods differ from ours in that the student maps directly between time points rather than parameterizing a continuous-time velocity field.

\paragraph{Adversarial distillation.}
Adversarial distillation trains a discriminator to provide distribution-level supervision for few-step generators. ADD~\citep{sauer2023adversarial} introduced this paradigm in pixel-space, and LADD~\citep{sauer2024ladd} moved the adversarial signal to latent space. Subsequent work combined adversarial losses with progressive distillation~\citep{lin2024sdxllightning} and integrated them into unified acceleration pipelines~\citep{chadebec2025flash}. While these methods introduce an explicit distributional signal, the additional discriminator network increases training complexity and can introduce instability.

\paragraph{Distribution matching distillation.}
Another line of work aligns generated distributions by comparing their noised score functions against those of a pretrained diffusion model. Since the noised score of the generated distribution is generally unavailable, SDS~\citep{poole2023dreamfusion} sidesteps this by using only the teacher score, while VSD~\citep{wang2024prolificdreamer}, Diff-Instruct~\citep{luo2024diff}, DMD~\citep{yin2024dmd}, and DMD2~\citep{yin2024dmd2} train auxiliary networks to approximate the student score. However, such learned proxies can lag behind the evolving student distribution, creating a moving-target problem.

In this paper, we augment reflow-based distillation with an explicit marginal-alignment objective, combining trajectory and marginal supervision while retaining a continuous-velocity student without requiring extra networks.
\section{Method}
\label{sec:method}

\subsection{Limitation of Trajectory Matching}
\label{sec:motivation}

Guided by the trajectory matching loss defined in \cref{eq:perflow}, 
PeRFlow imposes pointwise supervision on the student velocity field along interpolated teacher trajectories, thereby constraining the geometry of sample paths.
However, as highlighted by \cref{eq:instantaneous_change_of_variables}, the evolution of marginal densities along the flow depends on $\divop_x \vf$, a first-order spatial derivative of the velocity field, which is not constrained by the $L^2$ velocity error.
Consequently, two velocity fields with the same trajectory matching loss may lead to very different endpoint marginals as summarized by \cref{thm:underdetermination}.

\begin{theorem}[Trajectory matching cannot guarantee marginal alignment]
\label{thm:underdetermination}
There exist a teacher velocity field $\vg$ and two student velocity fields $\vec{f}_1, \vec{f}_2$ such that
\begin{equation}
    \cL_{\mathrm{TM}}^{(k)}(\vec{f}_1)=\cL_{\mathrm{TM}}^{(k)}(\vec{f}_2) \quad \text{for all } k=1,\ldots,K,
    \quad\text{but}\quad
    (\flow{\vec{f}_1}{0}{1})_{\#}\pi_0
    \neq
    (\flow{\vec{f}_2}{0}{1})_{\#}\pi_0.
\end{equation}
\end{theorem}

\cref{app:proof_thm5} provides the detailed proof via an explicit construction. 
For a constant teacher velocity field $\vg(x,t)\equiv\mu$, choose a unit vector $e$ with $\langle e,\mu\rangle=0$ and define two student velocity fields
\begin{equation}
    \vec{f}_1(x,t)=\mu+\varepsilon e
    \qquad \text{and} \qquad
    \vec{f}_2(x,t)=\mu+\varepsilon\langle x,e\rangle e,
\end{equation}
where $\langle \cdot,\cdot\rangle$ denotes the inner product and $\varepsilon>0$ is an arbitrary perturbation magnitude.
It can be shown that $\vec{f}_1$ and $\vec{f}_2$ have the same trajectory matching loss on every window.
However, $\vec{f}_1$ shifts the endpoint mean, whereas $\vec{f}_2$ dilates the endpoint covariance along $e$, leading to different KL divergences from their endpoint marginals to the teacher marginal at $t=1$.
This fact shows that the trajectory matching loss alone is insufficient to determine the model distribution at $t=1$ induced by the student velocity field.

\subsection{Reflow with Additional Marginal-Alignment Regularizer}
\label{sec:regularizer}
A natural idea to improve the defective trajectory matching loss is to enhance it by an extra regularizer that aligns the endpoint distribution induced by the student flow with the corresponding teacher marginal. 
To be concrete, let $\qpsi[,t_k] := (\flow{\vg}{0}{t_k})_{\#} \pi_0$ be the inference-time marginal at time point $t_k$ induced by the teacher velocity field $\vg$.
For each time window $W_k=[t_{k-1},t_k]$, define the \emph{training-protocol marginal} of the student velocity field $\vf$ at time point $t_k$ as
\begin{equation}
    \ptilde[^{(k)}] := (\flow{\vf}{t_{k-1}}{t_k})_{\#} \pi_{t_{k-1}},
    \label{eq:training_protocol_marginal}
\end{equation}
that is, the distribution obtained by pushing the reference noising marginal $\pi_{t_{k-1}}$ through the student flow over the $k$-th window. 
We then define the \emph{marginal-alignment loss} as 
\begin{equation}
    \Lma(\theta) := \sum_{k=1}^{K} \Lma^{(k)}(\theta)
    \qquad \text{where} \qquad \Lma^{(k)}(\theta) := \KL\!\left(\ptilde[^{(k)}] \,\big\|\, \qpsi[,t_k]\right)
    \label{eq:kl_def}
\end{equation}
measures the discrepancy between the student training-protocol marginal $\ptilde[^{(k)}]$ and the teacher endpoint marginal $\qpsi[,t_k]$. 
Enhancing the trajectory matching loss $ \cL_{\mathrm{TM}}$ with the marginal alignment loss $\Lma$, we get the following aggregated total loss with $\lambda\geq 0$ as hyper-parameter: 
\begin{equation}
    \cL_{\mathrm{total}}(\theta) :=  \cL_{\mathrm{TM}}(\theta) + \lambda \Lma(\theta).
    \label{eq:overall_objective}
\end{equation}
Through the diffusion--flow correspondence reviewed in \cref{app:diffusion_flow}, this marginal-alignment objective can also be applied to diffusion-model distillation under different noise schedules and prediction parameterizations.

Let $\ptheta[,t_k] := (\flow{\vf}{0}{t_k})_{\#} \pi_0$ 
be the inference-time marginal at time point $t_k$ induced by the student velocity field $\vf$.
The following theorem shows that minimizing $\Lma$ controls the model distributional discrepancy between student and teacher.
\begin{theorem}[$\Lma$ bounds the distributional gap]
\label{thm:tvbound}
Let $\delta_k= \TV(\qpsi[,t_k], \pi_{t_k})$ measure the gap between the teacher-induced marginal $\qpsi[,t_k]$ and the reference noising marginal $\pi_{t_k}$ in terms of the total variation distance, where $TV(\mu,\nu):=\sup_{A\in\cF}|\mu(A)-\nu(A)|$ with $\cF$ being the common sigma field of the two probability measures $\mu$ and $\nu$. 
Then under \cref{asmp:regularity,asmp:finite},
\begin{equation}
    \TV\!\left(\ptheta[,t_K], \qpsi[,t_K]\right) \leq \sum_{k=1}^{K} \sqrt{\tfrac{1}{2}\,\Lma^{(k)}(\theta)} + \sum_{k=1}^{K-1} \delta_k.
    \label{eq:tvbound}
\end{equation}
\end{theorem}

The proof, given in \cref{app:proof_thm4}, recursively tracks $\TV(\ptheta[,t_k],\qpsi[,t_k])$. At each window, the triangle inequality separates the propagated error from the current alignment error; data processing bounds the former, and Pinsker's inequality converts the latter from KL to TV. Telescoping then yields \cref{eq:tvbound}. The $\delta_k$ terms arise because training starts each window from the reference marginal $\pi_{t_{k-1}}$, whereas inference starts from the previous student marginal. Thus, reducing the proposed marginal-alignment loss tightens an upper bound on the final student--teacher distributional discrepancy.

\paragraph{Why the training-protocol marginal.}
One could instead try to minimize $\KL(\ptheta[,t_k] \| \qpsi[,t_k])$ directly. Here $\ptheta[,t_k] := (\flow{\vf}{0}{t_k})_{\#}\pi_0$ denotes the inference-time marginal induced by composing the student flows over the first $k$ windows. However, computing this objective and its gradient would require unrolling the sampler from $t_0$ to $t_k$, so each window update would depend on the full prefix of the sampling chain. The training-protocol marginal $\ptilde[^{(k)}]$ reduces this cost by localizing the objective to the $k$-th window. It uses the reference noising marginal $\pi_{t_{k-1}}$, which is independent of $\theta$, as the input distribution instead of the recursively generated marginal $\ptheta[,t_{k-1}]$. This choice remains tied to inference-time behavior through the telescoping bound in \cref{thm:tvbound}.

\paragraph{Role of trajectory matching.}
Marginal alignment is introduced as a complement to trajectory matching, not as a replacement for it. The PeRFlow objective provides trajectory-level supervision by matching the student velocity to the teacher-induced transport within each window, thereby distilling teacher ODE trajectories into local transports that are easier to execute with few-step sampling. The marginal-alignment loss then favors velocity fields that induce better density evolution among candidates with comparable trajectory error. This complementarity is nontrivial precisely when $\cL_{\mathrm{TM}} > 0$, which is the realistic regime under finite model capacity.

\subsection{Computation of the Marginal-Alignment Loss}
The following two identities, proved in \cref{app:proof_prop1,app:proof_prop2}, give a computable representation of $\Lma^{(k)}$ and its gradient.

\begin{lemma}[Liouville representation]
\label{prop:liouville}
Under the regularity conditions specified in \cref{app:assumptions},
\begin{equation}
    \Lma^{(k)}(\theta)
    =
    \E_{X \sim \pi_{t_{k-1}}}\!\left[
    \log \pi_{t_{k-1}}(X)
    + \DeltaStu_{t_{k-1}\to t_k}(X;\theta)
    - \log \qpsi[,t_k]\!\left(\Xstu_{t_{k-1}\to t_k}(X;\theta)\right)
    \right],
    \label{eq:liouville}
\end{equation}
where $\Xstu_{t_{k-1}\to t}(X;\theta) := \flow{\vf}{t_{k-1}}{t}(X)$ denotes the student trajectory at time $t$ initialized from $X$ at time $t_{k-1}$, and
$\DeltaStu_{t_{k-1}\to t_k}(X;\theta) := -\int_{t_{k-1}}^{t_k} \divop_x \vf\!\left(\Xstu_{t_{k-1}\to \tau}(X;\theta),\,\tau\right)\dd\tau$
is the corresponding log-density change.
\end{lemma}

\begin{theorem}[Gradient formula]
\label{prop:gradient}
For brevity, write
$\Xstu_k := \Xstu_{t_{k-1}\to t_k}(X;\theta)$ and
$\DeltaStu_k := \DeltaStu_{t_{k-1}\to t_k}(X;\theta)$.
Under the regularity conditions specified in \cref{app:assumptions},
the gradient of $\Lma^{(k)}$ can be written as
\begin{equation}
    \nabla_\theta \Lma^{(k)}(\theta)
    =
    \E_{X \sim \pi_{t_{k-1}}}\!\left[
    \nabla_\theta \DeltaStu_k
    -
    \spsi\!\left(\Xstu_k,t_k\right)^\top
    \nabla_\theta \Xstu_k
    \right],
    \label{eq:gradient}
\end{equation}
where $\spsi(x,t_k) := \nabla_x \log \qpsi[,t_k](x)$ is the teacher score function. Equivalently, the gradient in \cref{eq:gradient} is obtained by optimizing the following stop-gradient form:
\begin{equation}
    \E_{X \sim \pi_{t_{k-1}}}\!\left[
    \DeltaStu_k
    -
    \stopgrad\!\left(\spsi\!\left(\Xstu_k,t_k\right)\right)^\top \Xstu_k
    \right],
    \label{eq:stopgrad}
\end{equation}
where $\stopgrad(\cdot)$ is the standard deep-learning operator that leaves its argument unchanged in the forward pass and blocks gradients in the backward pass.
\end{theorem}

In practice, we obtain $\Xstu_{t_{k-1}\to t_k}(X;\theta)$ and $\DeltaStu_{t_{k-1}\to t_k}(X;\theta)$ by integrating the augmented student velocity field
\begin{equation}
    \vfaug((x,c),t) :=
    \begin{pmatrix}
        \vf(x,t) \\
        -\divop_x \vf(x,t)
    \end{pmatrix}
    \label{eq:augmented_ode}
\end{equation}
from the initial augmented state $(X,0)$ over $[t_{k-1},t_k]$. 
The divergence is estimated by the Hutchinson identity
$\divop_x \vf(x,t)=\E_{\epsilon \sim \N(0,I_d)}[\epsilon^\top \nabla_x \vf(x,t)\epsilon]$. Following the finite-difference implementation in JKO-iFlow~\citep{xu2023normalizing}, we approximate the directional derivative by
\begin{equation}
    \nabla_x \vf(x,t)\epsilon
    \approx
    \frac{\vf(x+\sigma\epsilon,t)-\vf(x,t)}{\sigma},
    \label{eq:fd_hutchinson}
\end{equation}
and estimate the resulting trace with Monte Carlo samples during training. Additional estimator and solver details are given in \cref{app:kl_details}. The frozen teacher score $\spsi$ is derived from the teacher velocity field $\vg$ via the velocity--score correspondence (see \cref{app:diffusion_flow} for details).

\Cref{alg:training} summarizes the training procedure of the proposed MA-Reflow method. In our experiments, we instantiate the trajectory-matching component with PeRFlow.
Since the marginal-alignment loss requires integrating the augmented student velocity field, we evaluate it every $m$ iterations and add it to the PeRFlow loss to reduce the computational cost of KL evaluation.

\begin{algorithm}[t]
\caption{Reflow with Additional Marginal-Alignment Regularizer (MA-Reflow). Instantiated with PeRFlow.}
\label{alg:training}
\begin{algorithmic}[1]
\renewcommand{\algorithmicrequire}{\textbf{Input:}}
\REQUIRE Training dataset $\mathcal{D}$, noise schedule $(a(t), b(t))$, teacher velocity field $\vg$ with associated score $\spsi$, window partition $\{t_k\}_{k=0}^K$, loss weight $\lambda$, alternating ratio $m$
\STATE Initialize the student model $\vf$
\REPEAT
    \STATE Sample data $X_1 \sim \mathcal{D}$ and noise $X_0 \sim \N(0,I_d)$
    \STATE Sample window $k \sim \mathrm{Uniform}\{1,\ldots,K\}$ and time $t \sim \mathrm{Uniform}(t_{k-1}, t_k)$
    \STATE Get $X = a(t_{k-1})X_0 + b(t_{k-1})X_1$
    \STATE $Y \leftarrow \flow{\vg}{t_{k-1}}{t_k}(X)$ \hfill // teacher endpoint
    \STATE $z_t \leftarrow X + \frac{t-t_{k-1}}{t_k-t_{k-1}}(Y-X)$
    \STATE $\ell \leftarrow \left\|\vf(z_t,t) - \frac{Y-X}{t_k-t_{k-1}}\right\|^2$ \hfill // PeRFlow loss
    \IF{every $m$ iterations}
        \STATE $(\Xstu_{t_{k-1}\to t_k}, \DeltaStu_{t_{k-1}\to t_k}) \leftarrow \flow{\vfaug}{t_{k-1}}{t_k}(X,0)$ \hfill // augmented student flow
        \STATE $\ell_{\mathrm{MA}} \leftarrow \DeltaStu_{t_{k-1}\to t_k} - \stopgrad(\spsi(\Xstu_{t_{k-1}\to t_k}, t_k))^\top \Xstu_{t_{k-1}\to t_k}$
        \STATE $\ell \leftarrow \ell + \lambda \ell_{\mathrm{MA}}$ \hfill // total loss
    \ENDIF
    \STATE Update $\theta$ using $\nabla_\theta \ell$
\UNTIL{convergence}
\end{algorithmic}
\end{algorithm}

\subsection{Comparison with Existing Methods}
\begin{table}[t]
\caption{Positioning of MA-Reflow relative to representative prior work.}
\label{tab:positioning}
\centering
\small
\begin{tabular}{ccccc}
\toprule
Methods & \makecell{Trajectory\\loss} & \makecell{Marginal\\loss} & \makecell{No extra\\network} & \makecell{Continuous\\velocity student} \\
\midrule
Reflow-based distillation & \checkmark & & \checkmark & \checkmark \\
Consistency distillation & \checkmark & & \checkmark & \\
Adversarial distillation & & \checkmark & & \\
Distribution matching distillation & & \checkmark & & \\
\textbf{MA-Reflow (ours)} & \checkmark & \checkmark & \checkmark & \checkmark \\
\bottomrule
\end{tabular}
\end{table}

\cref{tab:positioning} compares MA-Reflow with the distillation families reviewed in \cref{sec:related_works} along four axes: whether the method uses trajectory-level supervision, whether it introduces an explicit marginal-level objective, whether it avoids any auxiliary trainable network (e.g., critic, discriminator, or student-score estimator), and whether the student remains a continuous-time velocity field compatible with variable-step ODE solvers. MA-Reflow is the only method that combines trajectory and marginal supervision while retaining a continuous-velocity student and requiring no extra networks.
\section{Experiments}
\label{sec:experiments}

\subsection{Experimental Setup}
\label{sec:setup}

We evaluate the proposed marginal-alignment regularizer on two text-to-image backbones, Stable Diffusion v1.5 (SD~1.5)~\citep{rombach2022ldm} and Stable Diffusion XL (SDXL)~\citep{podell2023sdxl}. For both backbones, the student is initialized from a converged PeRFlow checkpoint~\footnote{We use the PeRFlow checkpoints released on Hugging Face: \href{https://huggingface.co/hansyan/perflow-sd15-delta-weights}{\texttt{hansyan/perflow-sd15-delta-weights}} for SD~1.5 and \href{https://huggingface.co/hansyan/perflow-sdxl-base}{\texttt{hansyan/perflow-sdxl-base}} for SDXL.} and further optimized with the joint objective in \cref{eq:overall_objective}. This setup reduces training cost and matches the role of our regularizer, which improves marginal alignment among students with comparable trajectory-matching quality. Additional experiments with a latent-space flow-matching teacher on CelebA-HQ, covering both teacher and PeRFlow initialization, are provided in \cref{app:lfm}.

\paragraph{Experimental configuration.}
For both SD~1.5 and SDXL experiments, training images are sampled from LAION-Aesthetics v2 5+~\citep{schuhmann2022laion}, following the PeRFlow training setup. Images are center-cropped and resized to $512{\times}512$ for SD~1.5 and $1024{\times}1024$ for SDXL. The distillation uses $K{=}4$ windows. Within each window, the $\epsilon$-prediction teacher produces the target endpoint using an 8-step DDIM sampler~\citep{song2021ddim}, which can be viewed as Euler integration in a rescaled coordinate system (\cref{app:vp_schedule}). The marginal-alignment regularizer is computed by solving the augmented ODE with a 4-step RK4 solver. To support classifier-free guidance (CFG), we apply 10\% text-conditioning dropout during training. The regularization weight $\lambda$ is set to $10^{-6}$. For efficiency, the marginal-alignment loss is evaluated once every nine PeRFlow updates, which accounts for roughly 30\% of the total training time in our implementation. All training runs are conducted on $2{\times}$NVIDIA H200 GPUs with BF16 mixed precision. Under this setup, SD~1.5 fine-tuning for 5K steps completes in approximately one day. For other hyperparameters, including learning rate and weight decay, we follow the Hugging Face training scripts for Stable Diffusion.\footnote{\href{https://github.com/huggingface/diffusers/tree/main/examples/text_to_image}{\texttt{diffusers/examples/text\_to\_image}}}
Following standard evaluation practice for text-to-image generation, we report Fr\'{e}chet Inception Distance (FID;~\citealp{heusel2017fid}) and CLIP score~\citep{radford2021learning}. Our evaluation uses the 2014 and 2017 validation sets of MS-COCO~\citep{lin2014coco}, following the BK-SDM~\citep{kim2023bksdm} setup for the 2014 validation set and the Rectified Diffusion~\citep{wang2025rectified} setup for the 2017 validation set.

\subsection{Main Results}
\label{sec:main_results}
\Cref{tab:sd15_coco,tab:sdxl} summarize the main quantitative results. Across both SD~1.5 and SDXL, the proposed method consistently improves few-step generation under the same 4-step inference budget. For SD~1.5, our method reduces COCO-2014 FID from 12.01 to 10.66 relative to PeRFlow and outperforms representative 4-step distillation baselines, including Flash Diffusion and PCM. Qualitative comparisons are provided in \cref{app:qualitative_sd15}. On COCO-2017, it also improves both evaluation metrics, reducing FID from 23.81 to 22.76 and increasing CLIP score from 30.24 to 31.03. The same trend carries over to SDXL, where our method improves over PeRFlow on both COCO-2017 and COCO-2014-10K. Among the compared 4-step methods, it achieves the best COCO-2014-10K FID and remains close to Rectified Diffusion on COCO-2017. These results suggest that the proposed regularizer improves the few-step generation quality of student models obtained through reflow-based distillation.

\begin{table}[t]
\centering
\small
\caption{SD~1.5 results on COCO-2014 and COCO-2017 validation sets.}
\label{tab:sd15_coco}
\setlength{\tabcolsep}{4.0pt}
\begin{tabular}{@{}c@{\hspace{0.08\textwidth}}c@{}}
\begin{minipage}[t]{0.38\textwidth}
\centering
\textbf{(a) COCO-2014 validation set.}\\[2pt]
\begin{tabular}{lcc}
\toprule
Method & Steps & FID ($\downarrow$) \\
\midrule
\multicolumn{3}{l}{\textit{Training-free ODE solvers}} \\
DPMSolver~\citep{lu2022dpm} & 25 & 9.78 \\
DPMSolver~\citep{lu2022dpm} & 8 & 22.44 \\
DPMSolver++~\citep{lu2025dpm} & 4 & 22.36 \\
DDIM~\citep{song2021ddim} (teacher) & 32 & 10.05 \\
\midrule
\multicolumn{3}{l}{\textit{Distilled few-step models}} \\
LCM-LoRA~\citep{luo2023lcm} & 4 & 23.62 \\
2-ReFlow~\citep{liu2022flow} & 4 & 15.32 \\
Flash Diffusion~\citep{chadebec2025flash} & 4 & 12.41 \\
PeRFlow~\citep{yan2024perflow} & 4 & 12.01 \\
PCM~\citep{wang2024pcm} & 4 & 11.70 \\
\textbf{MA-Reflow (ours, $K{=}4$)} & 4 & \textbf{10.66} \\
\bottomrule
\end{tabular}
\end{minipage}
&
\begin{minipage}[t]{0.47\textwidth}
\centering
\textbf{(b) COCO-2017 validation set.}\\[2pt]
\begin{tabular}{lccc}
\toprule
Method & Steps & FID ($\downarrow$) & CLIP ($\uparrow$) \\
\midrule
\multicolumn{4}{l}{\textit{Training-free ODE solvers}} \\
DPMSolver~\citep{lu2022dpm} & 25 & 20.10 & 31.80 \\
DDIM~\citep{song2021ddim} (teacher) & 32 & 21.98 & 31.22 \\
\midrule
\multicolumn{4}{l}{\textit{Distilled few-step models}} \\
LCM-LoRA~\citep{luo2023lcm} & 4 & 36.46 & 29.10 \\
2-ReFlow~\citep{liu2022flow} & 4 & 32.97 & 28.93 \\
LCM~\citep{luo2023latent} & 4 & 24.39 & 30.50 \\
PeRFlow~\citep{yan2024perflow} & 4 & 23.81 & 30.24 \\
\textbf{MA-Reflow (ours, $K{=}4$)} & 4 & \textbf{22.76} & \textbf{31.03} \\
\bottomrule
\end{tabular}
\end{minipage}%
\end{tabular}
\end{table}

\begin{table}[t]
\caption{SDXL results on COCO-2014-10K and COCO-2017 validation sets. Methods are sorted by FID on COCO-2014-10K.}
\label{tab:sdxl}
\centering
\small
\begin{tabular}{llcc}
\toprule
Method & Steps & \makecell{COCO-2014-10K\\FID ($\downarrow$)} & \makecell{COCO-2017\\FID ($\downarrow$)} \\
\midrule
SDXL-Lightning~\citep{lin2024sdxllightning} & 4 & 24.56 & 30.16 \\
SDXL-Turbo~\citep{sauer2023adversarial} & 4 & 23.19 & 30.92 \\
HyperSD~\citep{ren2024hypersd} & 4 & 22.95 & 30.58 \\
LCM~\citep{luo2023latent} & 4 & 22.16 & 27.09 \\
PCM~\citep{wang2024pcm} & 4 & 21.04 & --- \\
PeRFlow~\citep{yan2024perflow} & 4 & 20.99 & 27.06 \\
Rectified Diffusion~\citep{wang2025rectified} & 4 & 19.71 & \textbf{25.81} \\
DMD2~\citep{yin2024dmd2} & 4 & 19.32 & 26.19 \\
\textbf{MA-Reflow (ours, $K{=}4$)} & 4 & \textbf{19.17} & 25.94 \\
\bottomrule
\end{tabular}
\end{table}

\subsection{Sensitivity Analysis}
\label{sec:ablation}

We conduct sensitivity analyses on SD~1.5 to examine the behavior of the proposed method with respect to the CFG scale and the weight of the marginal-alignment regularizer. The results are summarized in \cref{fig:lambda}.

\paragraph{CFG sensitivity.}
CFG controls the balance between sample fidelity and text-image alignment during inference and can substantially affect the generation results of Stable Diffusion models. We sweep the guidance scale $w \in [2,7]$ on COCO-2017 under both 4-step and 8-step sampling. The resulting curves exhibit the standard trade-off between FID and CLIP score, with stable behavior across the tested range. This pattern indicates that the marginal-alignment regularizer is compatible with standard CFG-based inference rather than relying on a narrowly tuned guidance scale.

\paragraph{$\lambda$ sensitivity.}
The coefficient $\lambda$ controls the relative strength of marginal alignment in the joint objective. The default value $\lambda{=}10^{-6}$ is chosen according to the observed relative gradient magnitudes of the PeRFlow loss and the KL term. The right panel of \cref{fig:lambda} further examines nearby choices of $\lambda$ and shows improvements over the PeRFlow baseline for $\lambda \in [10^{-7}, 10^{-5}]$ under 4-step sampling, with the best result obtained at $\lambda{=}10^{-6}$. This sensitivity study supports the default scale used in the main experiments and suggests that the method does not require fine-grained tuning of $\lambda$.

\begin{figure}[t]
    \centering
    \includegraphics[width=1.0\textwidth]{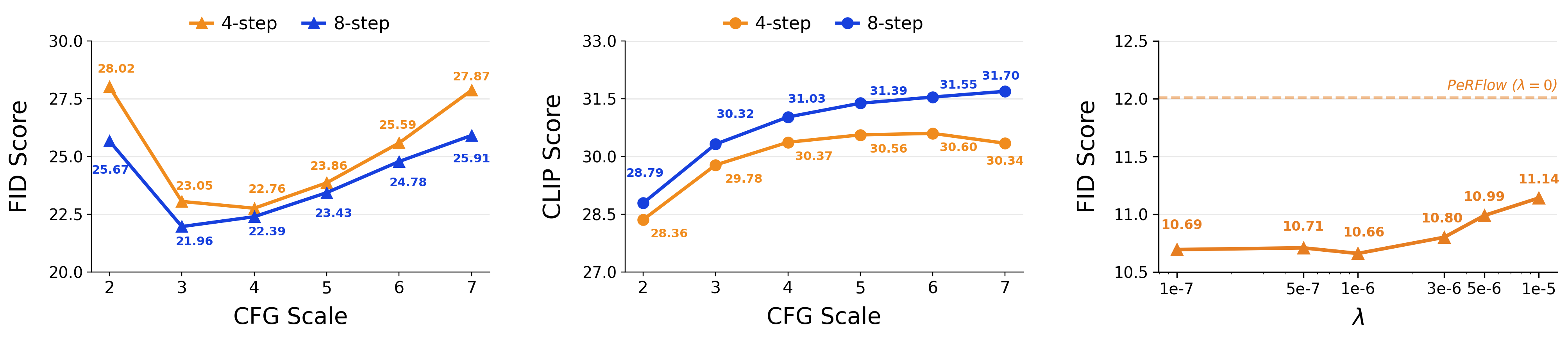}
    \caption{Sensitivity analysis on SD~1.5 (COCO-2017 validation).
    \textbf{Left column}: FID score under 4-step and 8-step sampling across CFG scales $w \in [2,7]$.
    \textbf{Middle column}: CLIP score under 4-step and 8-step sampling across CFG scales $w \in [2,7]$.
    \textbf{Right column}: FID score under 4-step across $\lambda$ values. The dashed line marks the PeRFlow baseline ($\lambda{=}0$).
    }
    \label{fig:lambda}
\end{figure}
\section{Conclusion}
\label{sec:conclusion}

In this work, we revisit reflow-based distillation from a distributional perspective. We show that trajectory matching, although effective for simplifying teacher ODE trajectories and improving the performance of few-step sampling, is insufficient to determine the distributions produced by the student when velocity matching is imperfect. This limitation motivates a complementary distribution-level signal. We therefore introduce a marginal-alignment regularizer that penalizes window-wise discrepancies from the corresponding teacher distributions without introducing auxiliary trainable networks. We further provide theoretical justification for this local-alignment strategy by connecting the window-wise penalties to the final-time distributional discrepancy. Empirically, the results across diffusion and flow-based settings indicate that marginal alignment is a useful complement to trajectory-based distillation for accelerating generative models.

\noindent\textbf{Limitations.}~~Although the marginal-alignment objective is compatible with other ODE-based reflow settings, including different noise schedules and the special case of $K{=}1$, empirical evaluation of these extensions is left to future work due to computational constraints. Extending the framework to non-ODE student models, such as consistency models or shortcut models, would require a different mechanism for assessing distribution-level discrepancies.

\bibliographystyle{plainnat}
\bibliography{references}

\newpage
\appendix
\numberwithin{table}{section}
\numberwithin{figure}{section}
\numberwithin{equation}{section}
\section{Proofs}
\label{app:proofs}

\subsection{Standing Assumptions}
\label{app:assumptions}

The following assumptions are used throughout the proofs. They ensure that the ODE systems are well posed and that the information-theoretic quantities in the analysis are finite.

\begin{assumption}[Regularity]
\label{asmp:regularity}
$\vf, \vg \in C^1(\R^d \times [0,1]; \R^d)$ are uniformly Lipschitz in $x$ with at most linear growth. Consequently, the ODE solutions exist globally, are unique, and induce $C^1$-diffeomorphic flow maps $\Phi$.
\end{assumption}

\begin{assumption}[Density regularity]
\label{asmp:density}
All distributions appearing in the analysis admit positive $C^1$ densities with sufficient decay at infinity, so that integration by parts introduces no boundary terms.
\end{assumption}

\begin{assumption}[Finiteness]
\label{asmp:finite}
All KL divergences and expectations involved in the analysis are finite.
\end{assumption}

\subsection{Proof of \texorpdfstring{\cref{thm:underdetermination}}{Theorem 1} (Trajectory matching cannot guarantee marginal alignment)}
\label{app:proof_thm5}

We prove \cref{thm:underdetermination} by constructing an explicit Gaussian example in which two student velocity fields attain the same trajectory-matching loss on every window, yet induce different endpoint marginals and different KL divergences to the teacher marginal.

\begin{proof}
Let $d \geq 2$, fix one window $[t_{k-1}, t_k]$, and let $\mu \in \R^d \setminus \{0\}$ be arbitrary. Define
\begin{equation}
    \pi_{t_{k-1}} = \N(\mu t_{k-1}, I_d), \qquad \pi_{t_k} = \N(\mu t_k, I_d),
\end{equation}
and let the teacher velocity field be $\vg(x,t) \equiv \mu$ on $[t_{k-1}, t_k]$. Then the teacher flow is $\flow{\vg}{t_{k-1}}{t_k}(x) = x + \mu(t_k - t_{k-1})$, the teacher-induced marginal is $\qpsi[,t_k] = \N(\mu t_k, I_d)$, and the PeRFlow target velocity is $\vec{v}_\psi^{\mathrm{avg}}(x) = \mu$.

Choose a unit vector $e \in \R^d$ such that $\langle e, \mu \rangle = 0$. For any $\varepsilon > 0$, define two student velocity fields on $[t_{k-1}, t_k]$ by
\begin{equation}
    \vec{f}_1(x,t) := \mu + \varepsilon e, \qquad \vec{f}_2(x,t) := \mu + \varepsilon \langle x, e \rangle e.
\end{equation}
Both fields satisfy \cref{asmp:regularity}.

\textbf{Step 1: Equal trajectory-matching losses.} Let $X \sim \pi_{t_{k-1}} = \N(\mu t_{k-1}, I_d)$. The teacher endpoint is $Y = X + \mu(t_k - t_{k-1})$, and the linear interpolation between $X$ and $Y$ is $X + (t - t_{k-1})\mu$. Since $\langle e, \mu \rangle = 0$, the inner product between this interpolation and $e$ is $\langle X, e \rangle$ for all $t$. The velocity errors relative to the target $\mu$ are
\begin{equation}
    \vec{f}_1(x,t) - \mu = \varepsilon e, \qquad \vec{f}_2(x,t) - \mu = \varepsilon \langle x, e \rangle e.
\end{equation}
Along the interpolation, $\|\vec{f}_1 - \mu\|^2 = \varepsilon^2$ and $\|\vec{f}_2 - \mu\|^2 = \varepsilon^2 \langle X, e \rangle^2$. Because $\langle X, e \rangle \sim \N(0, 1)$ by $\langle e, \mu \rangle = 0$, we have $\E[\langle X, e \rangle^2] = 1$. Therefore,
\begin{equation}
    \cL_{\mathrm{TM}}^{(k)}(\vec{f}_1) = \cL_{\mathrm{TM}}^{(k)}(\vec{f}_2) = \varepsilon^2 (t_k - t_{k-1}).
\end{equation}

\textbf{Step 2: Different global flow maps.} For $\vec{f}_1$, the dynamics are constant, $\dot{x} = \mu + \varepsilon e$. Thus each window adds the displacement $(\mu + \varepsilon e)(t_k - t_{k-1})$, and
\begin{equation}
    \flow{\vec{f}_1}{0}{1}(x) = x + \mu + \varepsilon e.
\end{equation}

For $\vec{f}_2$, the dynamics are linear, $\dot{x} = \mu + \varepsilon \langle x, e \rangle e$. Decompose $x = x_\parallel + x_\perp$, where $x_\parallel = \langle x, e \rangle e$. The coefficient $\eta(t) := \langle x(t), e \rangle$ satisfies $\dot{\eta}(t) = \varepsilon \eta(t)$, which gives $\eta(1) = e^{\varepsilon} \eta(0)$. Meanwhile, $x_\perp(1) = x_\perp(0) + \mu$. Combining these two components yields
\begin{equation}
    \flow{\vec{f}_2}{0}{1}(x) = x + \mu + (e^{\varepsilon} - 1)\langle x, e \rangle e.
\end{equation}

\textbf{Step 3: Different endpoint marginals at $t=1$.} Applying these maps to $X \sim \pi_0 = \N(0, I_d)$ gives
\begin{align}
    (\flow{\vec{f}_1}{0}{1})_{\#}\pi_0 &= \N(\mu + \varepsilon e,\; I_d), \\
    (\flow{\vec{f}_2}{0}{1})_{\#}\pi_0 &= \N(\mu,\; I_d + (e^{2\varepsilon} - 1)ee^\top).
\end{align}
Thus $\vec{f}_1$ shifts the endpoint mean while preserving the covariance, whereas $\vec{f}_2$ preserves the endpoint mean while inflating the covariance along $e$.

\textbf{Step 4: Different KL values.} With respect to $\qpsi[,1] = \N(\mu, I_d)$,
\begin{align}
    \KL((\flow{\vec{f}_1}{0}{1})_{\#}\pi_0 \,\|\, \qpsi[,1]) &= \frac{\varepsilon^2}{2}, \\
    \KL((\flow{\vec{f}_2}{0}{1})_{\#}\pi_0 \,\|\, \qpsi[,1]) &= \frac{1}{2}(e^{2\varepsilon} - 1 - 2\varepsilon) = \frac{1}{2}\sum_{n=2}^\infty \frac{(2\varepsilon)^n}{n!} > \frac{\varepsilon^2}{2}.
\end{align}
Therefore $\cL_{\mathrm{TM}}^{(k)}(\vec{f}_1) = \cL_{\mathrm{TM}}^{(k)}(\vec{f}_2)$ for all $k$, but $(\flow{\vec{f}_1}{0}{1})_{\#}\pi_0 \neq (\flow{\vec{f}_2}{0}{1})_{\#}\pi_0$.
\end{proof}

\subsection{\texorpdfstring{Proof of \cref{thm:tvbound} ($\Lma$ bounds the distributional gap)}{Proof of Theorem 2 (L\_MA bounds the distributional gap)}}
\label{app:proof_thm4}

We first establish a standard data-processing lemma for total variation.

\begin{lemma}[Data processing inequality for TV]
\label{lem:dpi}
For any measurable map $T: \R^d \to \R^d$ and probability measures $\mu, \nu$,
\begin{equation}
    \TV(T_\# \mu, T_\# \nu) \leq \TV(\mu, \nu).
\end{equation}
\end{lemma}

\begin{proof}
For any Borel set $A$, $T_\#\mu(A) = \mu(T^{-1}(A))$. Since $\{T^{-1}(A) : A \in \mathcal{B}(\R^d)\} \subseteq \mathcal{B}(\R^d)$,
\begin{align}
    \TV(T_\#\mu, T_\#\nu) &= \sup_A |T_\#\mu(A) - T_\#\nu(A)| \nonumber \\
    &= \sup_A |\mu(T^{-1}(A)) - \nu(T^{-1}(A))| \leq \TV(\mu, \nu). \qedhere
\end{align}
\end{proof}

\begin{proof}[Proof of \cref{thm:tvbound}]
Define $e_k := \TV(\ptheta[,t_k], \qpsi[,t_k])$.

\textbf{Base case.} Since the teacher and student share the initial distribution $\pi_0$, we have $e_0 = \TV(\pi_0, \qpsi[,0]) = \TV(\pi_0, \pi_0) = 0$.

\textbf{Inductive step.} For $k \geq 1$,
\begin{equation}
    e_k = \TV\!\left((\flow{\vf}{t_{k-1}}{t_k})_\# \ptheta[,t_{k-1}],\; \qpsi[,t_k]\right).
\end{equation}
Introduce $\ptilde[^{(k)}] = (\flow{\vf}{t_{k-1}}{t_k})_\# \pi_{t_{k-1}}$. By the triangle inequality,
\begin{equation}
    e_k \leq \underbrace{\TV\!\left((\flow{\vf}{t_{k-1}}{t_k})_\# \ptheta[,t_{k-1}],\; \ptilde[^{(k)}]\right)}_{\text{(I)}} + \underbrace{\TV\!\left(\ptilde[^{(k)}],\; \qpsi[,t_k]\right)}_{\text{(II)}}.
\end{equation}

\textbf{Bounding (I).} By \cref{lem:dpi},
\begin{equation}
    \text{(I)} \leq \TV(\ptheta[,t_{k-1}], \pi_{t_{k-1}}) \leq \TV(\ptheta[,t_{k-1}], \qpsi[,t_{k-1}]) + \TV(\qpsi[,t_{k-1}], \pi_{t_{k-1}}) = e_{k-1} + \delta_{k-1}.
\end{equation}

\textbf{Bounding (II).} By Pinsker's inequality,
\begin{equation}
    \text{(II)} \leq \sqrt{\tfrac{1}{2}\,\KL(\ptilde[^{(k)}] \| \qpsi[,t_k])} = \sqrt{\tfrac{1}{2}\,\Lma^{(k)}(\theta)}.
\end{equation}

Combining these two bounds gives
\begin{equation}
    e_k \leq e_{k-1} + \delta_{k-1} + \sqrt{\tfrac{1}{2}\,\Lma^{(k)}(\theta)}.
\end{equation}
Telescoping the recursion from $e_0 = 0$ yields
\begin{equation}
    e_K \leq \sum_{k=1}^K \sqrt{\tfrac{1}{2}\,\Lma^{(k)}(\theta)} + \sum_{k=0}^{K-1} \delta_k.
\end{equation}
Since $\delta_0 = \TV(\qpsi[,0], \pi_0) = 0$, the second sum equals $\sum_{k=1}^{K-1} \delta_k$.
\end{proof}

\subsection{Proof of \texorpdfstring{\cref{prop:liouville}}{Lemma 3} (Liouville representation)}
\label{app:proof_prop1}

The Liouville representation rewrites the KL divergence $\Lma^{(k)}$ as an expectation over $\pi_{t_{k-1}}$. This converts a distribution-level quantity into a sample-level expression that can be estimated by Monte Carlo.

\begin{proof}
By \cref{asmp:regularity}, the flow map $\flow{\vf}{t_{k-1}}{t_k}$ is a $C^1$-diffeomorphism. The change-of-variables formula gives
\begin{equation}
    \ptilde[^{(k)}](\flow{\vf}{t_{k-1}}{t_k}(x)) = \pi_{t_{k-1}}(x) \cdot \left|\det \nabla_x \flow{\vf}{t_{k-1}}{t_k}(x)\right|^{-1}.
\end{equation}
Taking logarithms and applying \cref{eq:instantaneous_change_of_variables} gives
\begin{align}
    \log \ptilde[^{(k)}]\!\left(\Xstu_{t_{k-1}\to t_k}(X;\theta)\right) &= \log \pi_{t_{k-1}}(X) - \int_{t_{k-1}}^{t_k} \divop_x \vf\!\left(\Xstu_{t_{k-1}\to \tau}(X;\theta), \tau\right)\dd\tau \nonumber \\
    & = \log \pi_{t_{k-1}}(X) + \DeltaStu_{t_{k-1}\to t_k}(X;\theta).
\end{align}
Substituting this identity into the definition of KL yields
\begin{align}
    \Lma^{(k)}(\theta) &= \E_{Y \sim \ptilde[^{(k)}]}\!\left[\log \ptilde[^{(k)}](Y) - \log \qpsi[,t_k](Y)\right] \nonumber\\
    &= \E_{X \sim \pi_{t_{k-1}}}\!\left[\log \pi_{t_{k-1}}(X) + \DeltaStu_{t_{k-1}\to t_k}(X;\theta) - \log \qpsi[,t_k]\!\left(\Xstu_{t_{k-1}\to t_k}(X;\theta)\right)\right],
\end{align}
where the second line changes variables from $Y = \Xstu_{t_{k-1}\to t_k}(X;\theta)$ back to $X \sim \pi_{t_{k-1}}$.
\end{proof}

\subsection{Proof of \texorpdfstring{\cref{prop:gradient}}{Theorem 4} (Gradient formula)}
\label{app:proof_prop2}

\begin{proof}
From \cref{prop:liouville}, the term $\log \pi_{t_{k-1}}(X)$ is independent of $\theta$. Hence,
\begin{equation}
    \nabla_\theta \Lma^{(k)}(\theta) = \E_{X \sim \pi_{t_{k-1}}}\!\left[\nabla_\theta \DeltaStu_{t_{k-1}\to t_k}(X;\theta) - \nabla_\theta \log \qpsi[,t_k]\!\left(\Xstu_{t_{k-1}\to t_k}(X;\theta)\right)\right].
\end{equation}
For the second term, the chain rule gives
\begin{equation}
    \nabla_\theta \log \qpsi[,t_k]\!\left(\Xstu_{t_{k-1}\to t_k}(X;\theta)\right) = \spsi\!\left(\Xstu_{t_{k-1}\to t_k}(X;\theta),\, t_k\right)^\top \nabla_\theta \Xstu_{t_{k-1}\to t_k}(X;\theta).
\end{equation}
The interchange of gradient and expectation follows from dominated convergence, since \cref{asmp:regularity} ensures that $\Xstu_{t_{k-1}\to t_k}(X;\theta)$ and $\DeltaStu_{t_{k-1}\to t_k}(X;\theta)$ have uniformly controlled $\theta$-derivatives.

Since $\spsi(\cdot, t_k)$ depends only on the fixed teacher parameters $\psi$ and not on $\theta$, the gradient does not propagate through $\spsi$. This gives the stop-gradient form in \cref{eq:stopgrad}.
\end{proof}

\section{Diffusion--Flow Correspondence and Score Derivation}
\label{app:diffusion_flow}
This section reviews the standard correspondence between diffusion models and continuous normalizing flows, covering arbitrary noise schedules and prediction parameterizations.
These identities both provide the velocity-to-score conversion needed to evaluate the regularizer and establish that it applies to arbitrary noise schedules and prediction parameterizations.
The correspondence reviewed below is well known and has been discussed by \citet{kingma2021variational}, \citet{karras2022edm}, \citet{albergo2023stochastic}, and \citet{kingma2023understanding}.
Throughout this section, $\theta$ denotes a generic network parameterization and does not necessarily refer to the student parameters used in the main text.

\subsection{General Framework}
\label{app:general_framework}

Let $Z_0 \sim \N(0,I_d)$ and $Z_1 \sim q_1$ be independent, and let $(a(t), b(t))$ be a noise schedule satisfying $a(0)=1, b(0)=0, a(1)=0, b(1)=1$. We write $\dot{a}(t) := \frac{\dd a(t)}{\dd t}$ and $\dot{b}(t) := \frac{\dd b(t)}{\dd t}$. The interpolation $X_t = a(t) Z_0 + b(t) Z_1$ for $t \in [0,1]$ induces a family of bridge distributions $\{p_t\}_{t\in[0,1]}$ between $\pi_0 = \N(0,I_d)$ and $q_1$, with
\begin{equation}
    p_t(x) = \int \N(x;\, b(t) x_1,\, a(t)^2 I_d)\, q_1(x_1)\,\dd x_1.
\end{equation}
These marginals are induced by a CNF with velocity field
\begin{equation}
    v^*(x,t) = \E[\dot{a}(t) Z_0 + \dot{b}(t) Z_1 \mid X_t = x].
    \label{eq:vstar}
\end{equation}

\paragraph{Velocity--score correspondence.}
Since $Z_0 = (X_t - b(t) Z_1)/a(t)$, substituting into \cref{eq:vstar} and applying the Tweedie identity $\E[Z_1 \mid X_t = x] = \frac{1}{b(t)}(x + a(t)^2 \nabla_x \log p_t(x))$ gives
\begin{equation}
    v^*(x,t) = \frac{\dot{b}(t)}{b(t)}\, x + a(t)^2 \left(\frac{\dot{b}(t)}{b(t)} - \frac{\dot{a}(t)}{a(t)}\right) \nabla_x \log p_t(x).
    \label{eq:velocity_score}
\end{equation}

\paragraph{Prediction parameterizations.}
Since $X_0 = Z_0$ and $X_1 = Z_1$, the velocity field can also be written as
\begin{equation}
    v^*(x,t) = \dot{a}(t) \,\underbrace{\E[Z_0 \mid X_t=x]}_{\epsilon\text{-prediction}} + \dot{b}(t)\, \underbrace{\E[Z_1 \mid X_t=x]}_{x_0\text{-prediction}}.
    \label{eq:prediction_types}
\end{equation}
Together with the identity $x = a(t) \E[Z_0 \mid X_t=x] + b(t) \E[Z_1 \mid X_t=x]$, any one of the velocity field, $\epsilon$-prediction, $x_0$-prediction, and the score function determines the other three. The differences among flow matching, $\epsilon$-prediction, $x_0$-prediction, and score matching therefore reduce to different weightings and reparameterizations of the same conditional expectation.

\subsection{VP Schedule Specialization}
\label{app:vp_schedule}

SD~1.5~\citep{rombach2022ldm} and SDXL~\citep{podell2023sdxl} both use $\epsilon$-prediction with a VP noise schedule satisfying $a(t) = \sigma(t)$, $b(t) = \alpha(t)$, and $\alpha^2(t) + \sigma^2(t) = 1$, so that the forward process takes the form $z_t = \alpha(t)\, z_0 + \sigma(t)\, \epsilon$. Writing $\phi(t) = \arcsin(\sigma(t))$, so that $\alpha(t) = \cos(\phi(t))$ and $\sigma(t) = \sin(\phi(t))$, the general identities in \cref{app:general_framework} specialize as follows.

\paragraph{Velocity field.}
From \cref{eq:prediction_types}, the velocity field corresponding to a network $\varepsilon_\theta$ that predicts $\E[Z_0 \mid X_t = x]$ is
\begin{equation}
    V_\theta(z_t, t) = \frac{\dot{\phi}(t)}{\alpha(t)}\left[\varepsilon_\theta(z_t, t) - \sigma(t)\, z_t\right].
    \label{eq:vp_velocity}
\end{equation}

\paragraph{Score function and $x_0$-prediction.}
From \cref{eq:velocity_score}, the score function is
\begin{equation}
    s_\theta(z_t, t) = -\frac{\varepsilon_\theta(z_t, t)}{\sigma(t)},
    \label{eq:vp_score}
\end{equation}
and the $x_0$-prediction is $X_\theta(z_t, t) = \frac{1}{\alpha(t)}[z_t - \sigma(t)\,\varepsilon_\theta(z_t, t)]$.

\paragraph{Backbone specifics.}
SD~1.5 and SDXL both use $\epsilon$-prediction with the same VP schedule described above, differing only in backbone architecture and default resolution.

\paragraph{DDIM as a rescaled Euler step.}
The velocity field in \cref{eq:vp_velocity} can be further simplified by a change of variables. Define the rescaled state $\bar{z}_t = z_t / \alpha(t)$ and the rescaled time $\gamma(t) = \sigma(t)/\alpha(t) = \tan(\phi(t))$. Differentiating $\bar{z}_t$ gives
\begin{equation}
    \mathrm{d}\bar{z}_t = \frac{1}{\alpha(t)}\bigl[V_\theta(z_t,t) + \dot{\phi}(t)\,\sigma(t)\,\bar{z}_t\bigr]\,\mathrm{d}t.
    \label{eq:rescaled_deriv}
\end{equation}
Substituting \cref{eq:vp_velocity} into \cref{eq:rescaled_deriv} and simplifying yields
\begin{equation}
    \mathrm{d}\bar{z}_t
    = \varepsilon_\theta(z_t, t)\;\mathrm{d}\gamma,
    \label{eq:ddim_ode}
\end{equation}
in which the $x_0$-prediction term cancels and the dynamics are driven solely by the noise prediction $\varepsilon_\theta$. A single Euler step of \cref{eq:ddim_ode} from $\gamma(r)$ to $\gamma(s)$, followed by the inverse rescaling $z_s = \alpha(s)\,\bar{z}_{\gamma(s)}$, recovers the standard DDIM
update~\citep{song2021ddim}:
\begin{equation}
    z_s
    = \alpha(s)\,\frac{z_r - \sigma(r)\,\varepsilon_\theta(z_r, r)}
                       {\alpha(r)}
      + \sigma(s)\,\varepsilon_\theta(z_r, r).
    \label{eq:ddim_update}
\end{equation}
Therefore, using the DDIM sampler to compute the teacher endpoint $Y = \flow{\vg}{t_{k-1}}{t_k}(X)$ in the training procedure (\cref{alg:training}, line~6) amounts to solving the teacher ODE with an Euler step in the $(\bar{z},\gamma)$ coordinate system.

\section{KL Estimator Details}
\label{app:kl_details}

This section provides implementation details for the marginal-alignment regularizer described in \cref{sec:method}. The augmented ODE system and the teacher score derivation are given in \cref{sec:method} and \cref{app:vp_schedule}, respectively. Below we describe the divergence estimator and the ODE solver used in practice.

\paragraph{Hutchinson trace estimator with finite differences.}
We estimate the divergence $\divop_x \vf$ in the augmented ODE using the Hutchinson identity. For any random vector $\epsilon \in \R^d$ satisfying $\E[\epsilon]=0$ and $\E[\epsilon\epsilon^\top]=I_d$,
\begin{equation}
    \divop_x \vf(x,t)
    =
    \Tr(\nabla_x \vf(x,t))
    =
    \E_{\epsilon}\!\left[\epsilon^\top \nabla_x \vf(x,t)\epsilon\right].
    \label{eq:hutchinson_identity}
\end{equation}
Following the finite-difference variant used in JKO-iFlow~\citep{xu2023normalizing}, we approximate the directional derivative by
\begin{equation}
    \nabla_x \vf(x,t)\epsilon
    \approx
    \frac{\vf(x+\sigma\epsilon,t)-\vf(x,t)}{\sigma},
    \label{eq:finite_difference_jvp}
\end{equation}
which yields the Monte Carlo estimator
\begin{equation}
    \divop_x \vf(x,t)
    \approx
    \frac{1}{N_e}\sum_{i=1}^{N_e}
    \epsilon_i^\top
    \frac{\vf(x+\sigma\epsilon_i,t)-\vf(x,t)}{\sigma}.
    \label{eq:hutchinson_fd}
\end{equation}
Here $\epsilon_i \sim \N(0,I_d)$ are i.i.d.\ samples, and $\sigma=\sigma_0/\sqrt d>0$ is the finite-difference scale. In our experiments, we use $N_e=1$, with $\sigma_0=10$ for SD~1.5 and $\sigma_0=1$ for SDXL.

\paragraph{ODE solver.}
In each time window $W_k=[t_{k-1},t_k]$, we solve the augmented ODE with a 4-step fixed-step-size RK4 solver and compute gradients via the adjoint method~\citep{chen2018neural}.

\section{Additional Experimental Results}
\label{app:additional_results}

This section reports supplementary experiments that complement the main results in \cref{sec:experiments}. We first consider a latent-space flow-matching (LFM) model pretrained on CelebA-HQ (\cref{app:lfm}) to examine the marginal-alignment regularizer under different initialization conditions, and then provide qualitative comparisons on SD~1.5 (\cref{app:qualitative_sd15}).

\subsection{LFM model on CelebA-HQ}
\label{app:lfm}

The main experiments in \cref{sec:experiments} initialize the student from a converged PeRFlow checkpoint. Here we conduct complementary experiments on CelebA-HQ $256{\times}256$ with the pretrained LFM model of \citet{dao2023flow}, which directly predicts velocity using a UNet architecture. The student adopts the same architecture and parameterization. We consider two settings: initializing from the teacher checkpoint (Setting~A) and initializing from a converged PeRFlow checkpoint (Setting~B). As a reference, the pretrained teacher achieves an FID of 7.35 with its default 32-step Euler sampler.

\paragraph{Setting A: Training from teacher initialization.}
Starting from the pretrained LFM checkpoint, we train PeRFlow ($K{=}4$) and MA-Reflow ($K{=}4$) respectively for 100K iterations with identical hyperparameters, except that MA-Reflow additionally uses the marginal-alignment loss ($\lambda{=}0.1$, evaluated once every nine PeRFlow updates). \Cref{tab:lfm_a} reports the best FID achieved during training for each method.

\begin{table}[ht]
\centering
\caption{LFM on CelebA-HQ: training from teacher initialization (Setting~A). Best FID over training. ``LFM (direct)'' denotes the pretrained teacher sampled with the corresponding number of steps without distillation.}
\label{tab:lfm_a}
\small
\begin{tabular}{lccc}
\toprule
Steps & LFM (direct) & PeRFlow & MA-Reflow \\
\midrule
4-step & 54.41 & 18.17 & \textbf{17.25} \\
8-step & 21.29 & \textbf{12.87} & 13.04 \\
\bottomrule
\end{tabular}
\end{table}

\paragraph{Setting B: Fine-tuning from converged PeRFlow.}
Starting from the best PeRFlow checkpoint in Setting~A, we continue training with MA-Reflow and, as a control, with PeRFlow for an additional 100K iterations. \Cref{tab:lfm_b} reports the best FID achieved during training. Continuing with PeRFlow degrades FID, suggesting that further trajectory-matching updates no longer improve generation quality at this checkpoint. In contrast, MA-Reflow further improves FID under both 4-step and 8-step sampling.

\begin{table}[ht]
\centering
\caption{LFM on CelebA-HQ: fine-tuning from converged PeRFlow (Setting~B). Best FID over training. ``PeRFlow (init)'' denotes the starting checkpoint from Setting~A.}
\label{tab:lfm_b}
\small
\begin{tabular}{lcccc}
\toprule
Steps & LFM (direct) & PeRFlow (init) & PeRFlow (cont.) & MA-Reflow \\
\midrule
4-step & 54.41 & 18.17 & 23.87 & \textbf{16.78} \\
8-step & 21.29 & 12.87 & 13.85 & \textbf{11.98} \\
\bottomrule
\end{tabular}
\end{table}

Taken together, the two settings show that marginal-alignment regularization provides consistent benefits under different initialization conditions, in line with the complementarity between trajectory matching and marginal alignment identified in \cref{thm:underdetermination,thm:tvbound}.

\subsection{Qualitative Comparison on SD~1.5}
\label{app:qualitative_sd15}

\begin{figure}[htbp]
    \centering
    \includegraphics[width=0.635\linewidth]{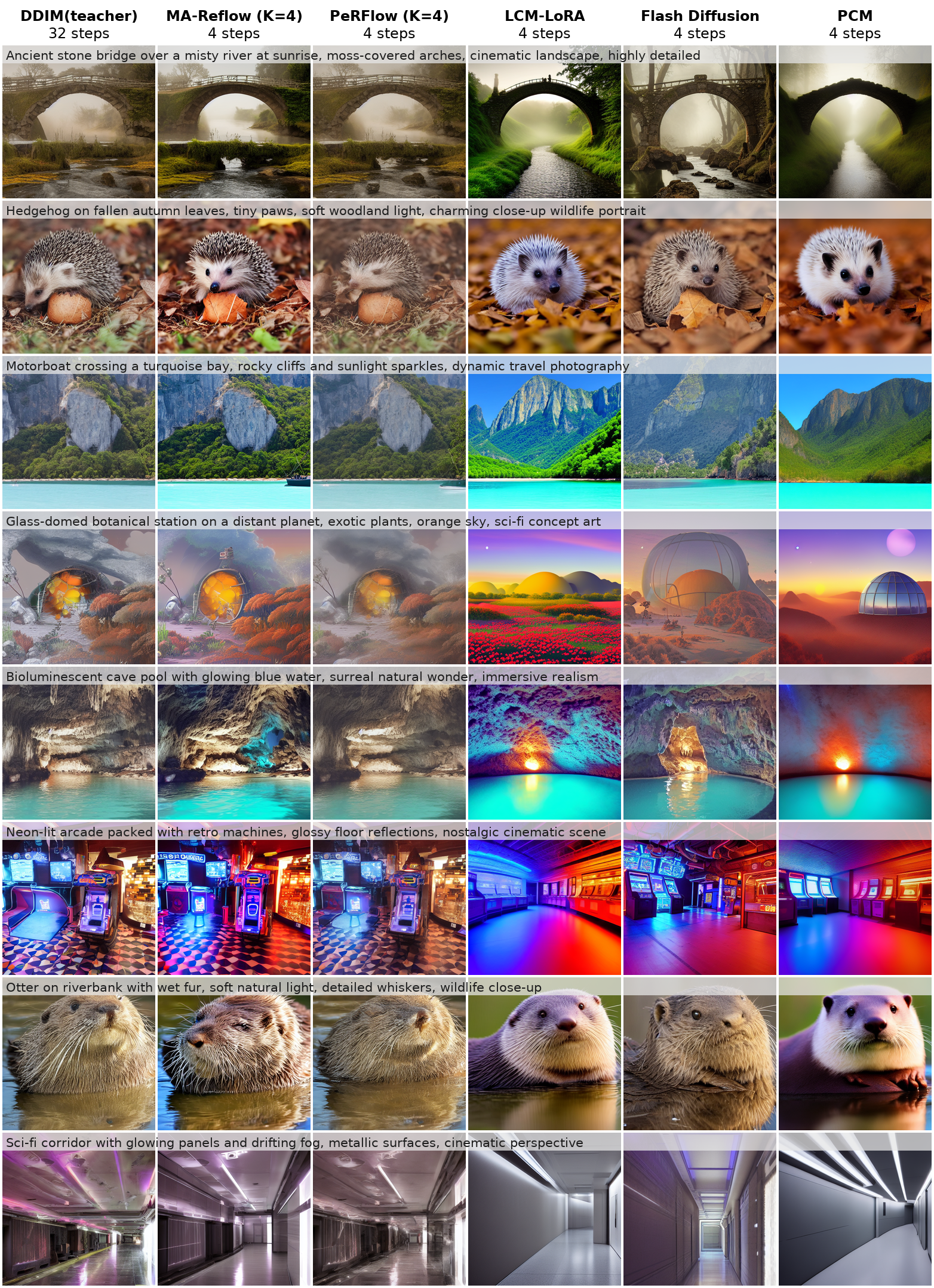}
    \caption{Qualitative comparison on SD~1.5 under 4-step sampling. Each row shares the same prompt and initial noise. Our method achieves better visual quality and preserves richer fine-grained details than other distillation baselines.}
    \label{fig:qualitative_sd15}
\end{figure}

\section{Broader Impacts}
\label{app:broader_impacts}

This work studies efficient distillation methods for diffusion and flow-based generative models. Our experiments cover both unconditional image generation and conditional text-to-image synthesis, though the proposed approach may extend to other modalities such as audio and video. By reducing the number of sampling steps required for high-quality generation, our method can lower the computational cost and energy consumption of deploying these models. At the same time, more efficient generators could lower the barrier to producing disinformation or non-consensual imagery. This risk is shared by generative-model acceleration methods in general, and existing countermeasures such as output watermarking, provenance tracking, and synthetic-content detection remain fully applicable. All our experiments build on publicly available pretrained models and datasets.

\end{document}